\documentclass{article}

\usepackage{microtype}
\usepackage{graphicx}
\usepackage{subfigure}
\usepackage{booktabs} 
\usepackage{printlen}
\usepackage{comment}

\usepackage{amsmath}
\usepackage{amsfonts}
\usepackage{algorithm}
\usepackage{algorithmic}
\usepackage{tikz}
\usepackage{multirow}
\usepackage{multicol}
\usepackage{tablefootnote}
\usepackage{float}
\usepackage{caption}
\usepackage{layouts}
\usepackage{wrapfig}
\usetikzlibrary{matrix, positioning}
\usetikzlibrary{positioning, shapes.geometric}

\usepackage{hyperref}

\usepackage[accepted]{icml2024}

\usepackage{amsmath}
\usepackage{amssymb}
\usepackage{mathtools}
\usepackage{amsthm}
\usepackage{svg}
\usepackage{graphicx}
\usepackage{mwe}

\usepackage[capitalize,noabbrev]{cleveref}

\theoremstyle{plain}
\newtheorem{theorem}{Theorem}[section]
\newtheorem{proposition}[theorem]{Proposition}

\theoremstyle{definition}

\theoremstyle{remark}

\usepackage[textsize=tiny]{todonotes}
\usepackage{svg-extract}

\icmltitlerunning{ICML 2025 Workshop on Machine Unlearning for Generative AI}

\begin{document}
\twocolumn[
\icmltitle{ContinualFlow: Learning and Unlearning with Neural Flow Matching}



\icmlsetsymbol{equal}{*}

\begin{icmlauthorlist}
\icmlauthor{Lorenzo Simone}{pi,yale}
\icmlauthor{Davide Bacciu}{pi}
\icmlauthor{Shuangge Ma}{yale}
\end{icmlauthorlist}

\icmlaffiliation{pi}{Department of Computer Science, University of Pisa, Pisa, Italy}
\icmlaffiliation{yale}{School of Public Health, Yale University, New Haven, U.S.A.}

\icmlcorrespondingauthor{Lorenzo Simone }{lorenzo.simone@yale.edu}

\icmlkeywords{Generative Models, Flow Matching, Unlearning, Energy-Based Models, Optimal Transport}

\vskip 0.3in
]



\printAffiliationsAndNotice{} 

\begin{abstract}
We introduce \emph{ContinualFlow}, a principled framework for targeted unlearning in generative models via Flow Matching. Our method leverages an energy-based reweighting loss to softly subtract undesired regions of the data distribution without retraining from scratch or requiring direct access to the samples to be unlearned. Instead, it relies on energy-based proxies to guide the unlearning process. We prove that this induces gradients equivalent to Flow Matching toward a soft mass-subtracted target, and validate the framework through experiments on 2D and image domains, supported by interpretable visualizations and quantitative evaluations.
\end{abstract}

\section{Introduction}
Machine unlearning, the removal of specific information from trained machine learning (ML) models, has moved from a niche technical concern to a central issue at the intersection of law, ethics, and model deployment. The recent widespread adoption and training of image generation and large language models (LLMs) has significantly expanded the scope and stakes of the field \citep{cooper2024machine}.

Adapting machine unlearning to generative learning introduces distinct conceptual and technical challenges. While discriminative models often allow data traces to be linked directly to outputs, generative models learn complex mappings from latent or prior distributions to data, resulting in entangled and opaque representations. This makes it difficult to isolate the influence of specific inputs, and even when behavior is altered, the model may still output related content through prompting \citep{gpt_editing} or interpolation \citep{interpolation}. In this context, the notion of content erasure remains ambiguous, underscoring the need for precise definitions and tools tailored to generative settings.

\begin{figure}[t]
    \centering
    \includegraphics[width=0.5\textwidth]{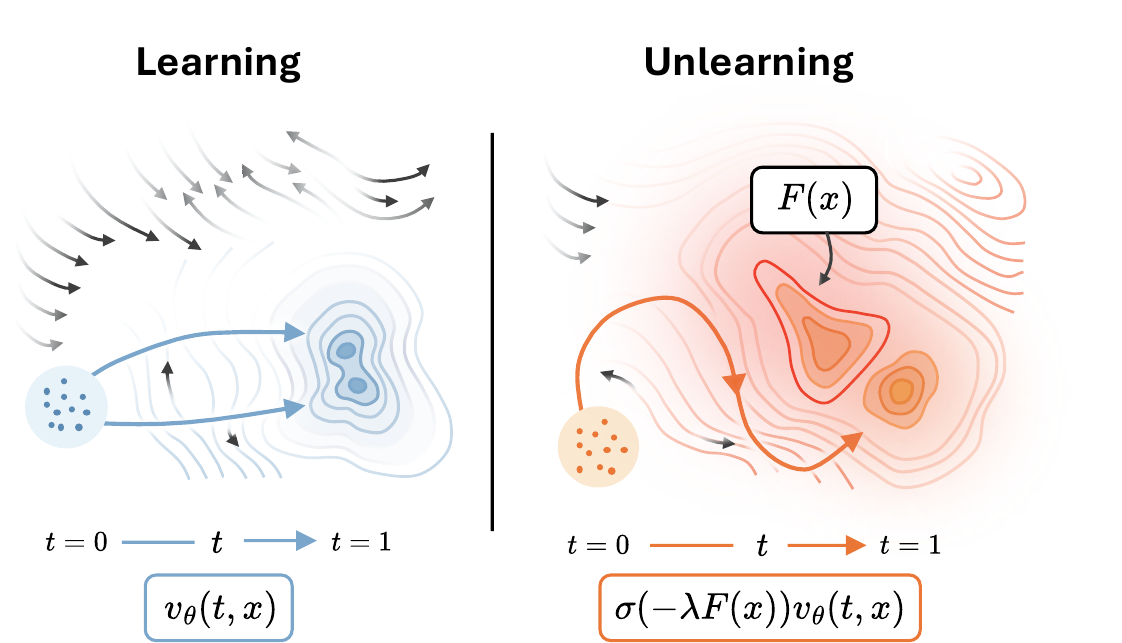}
    \caption{\textbf{Conceptual overview of ContinualFlow.}
\textit{Left:} Standard learning via Flow Matching, where a neural vector field \( v_\theta(t, x) \) maps a base distribution to a known target. \textit{Right:} Energy-guided unlearning, where the flow is softly modulated by a proxy \( \sigma(-\lambda F(x)) \) to steer trajectories away from undesirable regions without access to samples or exact densities.}
    \label{fig:method}
\end{figure}

Recent research on generative machine unlearning primarily focuses on two strategies: (1) output suppression and (2) model patching.

\emph{Output suppression} restricts undesired content at generation time without altering internal representations. For instance, text-to-image models can use guided decoding or inference-time filters \citep{gandikota2023erasing} to avoid producing sensitive outputs. While effective at steering generation, such methods do not remove the underlying knowledge and can often be bypassed with adversarial prompts.

\emph{Model patching} instead modifies model parameters via fine-tuning or targeted updates. This approach offers more durable unlearning, such as preventing a model from reproducing a specific content, and is less susceptible to adversarial prompting. However, generative models may still reproduce conceptually similar samples via alternate pathways, and patching can inadvertently impair unrelated abilities---underscoring the trade-off between effective forgetting and preserving overall model generalization \citep{liu2025rethinking}. For a broader taxonomy of generative unlearning methods and their limitations, see Appendix~\ref{apd:related_work}.


While most unlearning strategies rely on inference-time filters or post-hoc model editing, recent advances in generative modeling open new possibilities for controlling how distributions are learned---and unlearned---through trajectory-based formulations. In particular, Flow Matching (FM)~\citep{lipmanflow} frames generation as a continuous-time transport process, where samples evolve along learned trajectories defined by a neural velocity field. 

In this work, we explore how this geometric formulation can be adapted for unlearning by integrating it with energy functions, which act as scalar potentials assigning higher values to inputs linked to undesirable content. To enable sample-independent unlearning, we formulate Energy-Reweighted Flow Matching (ERFM) as a theoretically grounded extension of the flow matching framework, incorporating energy-derived weights into the training objective. These weights softly downregulate high-risk regions in the data space, effectively steering generative trajectories away from undesired content (Figure \ref{fig:method}).

Beyond their role in defining target regions for suppression, energy functions offer theoretical advantages that support future extensions. In particular, their modularity enables compositionality: distinct objectives can, in principle, be combined to encode evolving unlearning criteria. We further evaluate, through targeted experiments, how the invertibility of the energy function influences flow behavior, demonstrating its impact on the design of adaptive unlearning mechanisms.

\paragraph{Contributions.} Our main contributions are threefold: (1) we introduce a principled extension of Flow Matching that integrates energy functions as soft proxies for unlearning, enabling attenuation of generation in high-risk regions without requiring explicit forget samples; (2) we instantiate this formulation in \emph{ContinualFlow}, a modular framework that supports both learning and unlearning via energy-guided updates of generative flows; and (3) we evaluate our method across both 2D and image-based benchmarks, demonstrating effective unlearning with minimal impact on generation fidelity and training efficiency.

\section{Background: Flow Matching and EBMs}
\paragraph{Continuous-Time Generative Flows.} Flow Matching is a paradigm for training continuous-time generative models by learning a velocity field \( u : [0,1] \times \mathbb{R}^d \rightarrow \mathbb{R}^d \) that defines a transport from a tractable base distribution \( p_0 \) to a target data distribution \( p_1 \). The generative process is modeled as a deterministic flow \( \phi_t(x) \) defined by the ordinary differential equation (ODE):
\begin{equation}
\frac{d}{dt} \phi_t(x) = u(t, \phi_t(x)), \quad \phi_0(x) = x,
\end{equation}
such that \( \phi_1(X_0) \sim p_1 \) when \( X_0 \sim p_0 \).

To make this framework trainable in practice, recent work introduces \emph{Conditional Flow Matching} (CFM) framework~\citep{tong2023improving}, which reformulates the learning of \( u(t, x) \) as a supervised regression task over a family of conditional displacements. In its original form, CFM trains a neural velocity field \( v_\theta(t, x) \) to approximate a prescribed conditional velocity \( u_t(x \mid z) \) via the objective:

\begin{equation*}
\resizebox{0.95\linewidth}{!}{$
\mathcal{L}_{\text{CFM}}(\theta) = \mathbb{E}_{t \sim \mathcal{U}[0,1],\, z \sim q(z),\, x \sim p_t(x \mid z)} \left[ \left\| v_\theta(t, x) - u_t(x \mid z) \right\|^2 \right],$
}
\end{equation*}

where \( z = (x_0, x_1) \sim p_0 \times p_1 \), and the components are defined as:
\begin{small}
\begin{align*}
\psi_t(x_0, x_1) &= (1 - t)x_0 + t x_1 
  \tag*{\text{\scriptsize [Interpolation path]}} \\
p_t(x \mid x_0, x_1) &= \mathcal{N}(x \mid \psi_t(x_0, x_1), \sigma^2 \mathbb{I}) 
  \tag*{\text{\scriptsize [Conditional distribution]}} \\
u_t(x \mid x_0, x_1) &= x_1 - x_0 
  \tag*{\text{\scriptsize [Target velocity]}}
\end{align*}
\end{small}

Rather than directly adopting this approach, we build upon its formulation to derive a modified objective tailored to our unlearning setting.

\paragraph{Energy Functions.} A standard approach to modeling unnormalized distributions is the Boltzmann energy-based representation:
\begin{equation}
    p(x) = \frac{1}{Z} \exp(-F(x)),
\end{equation}
where \( F(x): \mathbb{R}^d \rightarrow \mathbb{R} \) is a scalar energy function, and \( Z = \int \exp(-F(x)) dx \) is the partition function ensuring normalization. Since computing \( Z \) is generally intractable, training often relies on alternatives such as score matching~\citep{hyvarinen2005estimation} or contrastive divergence~\citep{hinton2005}. Sampling from these unnormalized models is particularly challenging in high dimensions and is typically addressed using approximate methods such as Langevin dynamics~\citep{welling2011bayesian}. 

\section{Problem Setting: Learning and Unlearning}

We consider the task of removing the influence of specific data from a trained generative model \( G_\theta \). Let \mbox{\( \mathcal{D}_{\text{full}} = \mathcal{D}_{\text{retain}} \cup \mathcal{D}_{\text{forget}} \)} 
denote the full training dataset, where 
\mbox{\( \mathcal{D}_{\text{retain}} \cap \mathcal{D}_{\text{forget}} = \emptyset \)}. Assume a generative model has been trained on \(\mathcal{D}_{\text{full}}\). We analyze two distinct settings, assuming different degrees of access to $\mathcal{D}_{\text{forget}}$.

\subsection{Case 1: Sample-Based Unlearning with Full Access}

In this setting, we assume direct access to the forget set \(\mathcal{D}_{\text{forget}}\) and aim to erase its influence while preserving performance on \(\mathcal{D}_{\text{retain}}\). Training from scratch typically involves learning from a standard Gaussian prior or an exact density to match  \(\mathcal{D}_{\text{retain}}\), which can be both inefficient and unnecessary. Instead, we show that \emph{Optimal Transport Flow Matching} (OT-FM)~\citep{tong2023improving} can be used to directly model the transition between samples generated by the original model \( G_\theta \) and the retained data \(\mathcal{D}_{\text{retain}}\). This strategy removes the dependency on a predefined prior and simplifies training by focusing on the actual optimal shift in distribution. Further theoretical justification for this approach is provided in Appendix~\ref{apd:otcfm}.

\subsection{Case 2: Unlearning Without Forget Set Access}

Unlike the previous setting---where \(\mathcal{D}_{\text{forget}}\) is explicitly available---this case addresses a more realistic and challenging scenario that is the primary focus of our work. In practice, direct access to such data is often infeasible, particularly in settings involving privacy regulations or dynamic content updates, where new instances of sensitive data may continuously emerge but are not explicitly labeled. Instead, it is far more common to rely on proxy functions, such as classifiers or scoring models, trained to detect undesirable content based on prior knowledge. We treat their output as an unnormalized energy function \(F(x) \propto -\log p(x)\), enabling principled, sample-free updates of generative trajectories without requiring normalization constants or explicit access to the forget distribution.

\section{ContinualFlow: Proposed Methodology}
\label{sec:method}

We introduce a principled transport-based formulation for generative unlearning, leveraging trajectory-level modulation for sensitive content via energy-guided updates.

\subsection{Soft Mass Subtraction via Energy Reweighting}

Let \( q_0(x) \) denote a known, samplable source distribution, and let \( q_f(x) \) represent an unknown distribution over data to be forgotten. We assume access to a scalar energy function \( F(x) \propto -\log q_f(x) \), which scores space configurations based on their association with the forget distribution.

We construct a reweighted surrogate of \( q_0(x) \) by modulating its density as
\[
\tilde{R}(x) \propto q_0(x) \cdot \sigma(-\lambda F(x)),
\]
where \( \sigma(z) = \frac{1}{1 + e^{-z}} \) is the sigmoid function and \( \lambda > 0 \) controls the suppression sensitivity. The term \( \sigma(-\lambda F(x)) \) smoothly downweights high-energy regions, reducing the influence of samples likely originating from \( q_f \).

The corresponding normalized target distribution is:
\[
\tilde{q}_1(x) = \frac{1}{Z} \tilde{R}(x), \quad \text{with} \quad Z = \int q_0(x)\, \sigma(-\lambda F(x)) \, dx.
\]

We interpret \( \tilde{q}_1 \) as the terminal marginal of a generative flow that selectively suppresses regions aligned with \( q_f \), without requiring direct access to \( q_f \) itself. This yields a continuous, differentiable relaxation for unlearning via smooth density reweighting rather than thresholding or sample exclusion.

\subsection{Energy-Reweighted Flow Matching Objective}

We now formalize the training objective that aligns a pretrained generative model on \( \mathcal{D}_{\text{full}} \) with the reweighted target distribution \( \tilde{q}_1(x) \). Unlike inference-time approaches such as classifier guidance~\citep{dhariwal2021diffusion}—which steer generation away from undesired regions using external gradients while keeping the original model unchanged—our method directly modifies the training loss by reweighting samples based on their association with the forget distribution. This reweighting shapes the velocity field to avoid high-energy regions and favor trajectories toward retained content. 

The following theorem shows that our objective is equivalent (up to a constant) to CFM toward \( \tilde{q}_1(x) \), using only samples from the base distribution \( q_0 \).

\begin{theorem}
\label{thm:erfm_equiv}
Let \( q_0 \) be a base distribution and \( F(x) \propto -\log q_f(x) \) an energy function for an unknown forget distribution \( q_f \). Define the reweighted target as \( \tilde{q}_1(x) \propto q_0(x) \cdot \sigma(-\lambda F(x)) \), with \( \lambda > 0 \). Then, the Energy-Reweighted Flow Matching loss
\begin{multline}
\mathcal{L}_{\mathrm{ERFM}}(\theta) =
\mathbb{E}_{x_0, x_1 \sim q_0,\,
t \sim \mathcal{U}[0,1],\,
x \sim p_t(x \mid x_0, x_1)}
\\
\left[ \sigma(-\lambda F(x_1)) \cdot
\left\| v_\theta(t, x) - u_t(x \mid x_0, x_1) \right\|^2 \right]
\end{multline}
satisfies
\[
\nabla_\theta \mathcal{L}_{\mathrm{ERFM}}(\theta)
= C \cdot \nabla_\theta \mathcal{L}_{\mathrm{CFM}}^{q_0 \to \tilde{q}_1}(\theta)
\]
for some constant \( C > 0 \).
\end{theorem}

Each training pair \( (x_0, x_1) \sim q_0 \times q_0 \) is weighted by \( \sigma(-\lambda F(x_1)) \), yielding an importance-weighted objective that guides the flow away from high-energy regions. This enables efficient training toward \( \tilde{q}_1 \) without requiring explicit access to \( q_f \). See Algorithm~\ref{alg:continualflow} for the full procedure.


\vspace{0.02cm}
\begin{algorithm}[hbtp]
    \caption{\textit{ContinualFlow}: Training procedure}
    \label{alg:continualflow}
    \begin{footnotesize}
    \begin{algorithmic}[1]
        \STATE \textbf{Input:} Initial distribution $q_0(x)$, energy $F(x)$, model $v_\theta$, steps $S$, batch size $B$, learning rate $\eta$, scale $\lambda$
        
        \STATE \textbf{Training:}
        \FOR{$i = 1 \text{ to } S$}
            \STATE Sample $\{x_0^{(j)}\}_{j=1}^B \sim q_0$;  $\{x_1^{(j)}\}_{j=1}^B \sim q_0$
            \STATE Sample $\{t^{(j)}\}_{j=1}^B \sim \mathcal{U}(0,1)$
            \STATE Interpolate $x_t^{(j)} = (1 - t^{(j)}) x_0^{(j)} + t^{(j)} x_1^{(j)}$
            \STATE Define weights $w^{(j)} = \sigma(-\lambda F(x_1^{(j)}))$
            \STATE Define targets $\Delta x^{(j)} = x_1^{(j)} - x_0^{(j)}$
            \STATE Compute loss: 
            \[
             \mathcal{L} = \frac{\sum_{j=1}^B w^{(j)} \|v_\theta(x_t^{(j)}, t^{(j)}) - \Delta x^{(j)}\|^2}{\sum_{j=1}^B w^{(j)}}
            \]
            \STATE Update $\theta \leftarrow \theta - \eta \nabla_\theta \mathcal{L}$
        \ENDFOR
    \end{algorithmic}
    \end{footnotesize}
\end{algorithm}

\begin{table*}[t]
\centering
\caption{Comparison of unlearning performance across methods and datasets.}
\label{tab:results}
\resizebox{0.77\textwidth}{!}{
\begin{tabular}{llccccc}
\toprule
\textbf{Dataset} & \textbf{Method} 
& \multicolumn{2}{c}{\textbf{Retention}} 
& \multicolumn{2}{c}{\textbf{Forgetting}} 
& \multicolumn{1}{c}{\textbf{Efficiency}} \\
& & MMD ($\downarrow$) & Accuracy ($\uparrow$) 
  & Forget Rate ($\downarrow$) & Leakage ($\downarrow$) 
  & Train Time (s) \\
\midrule

\multirow{3}{*}{2D} 
& Retrain (GT)    & 0.0157 $\pm$ 0.0104 & 0.9999 $\pm$ 0.0002 & 0.0007 $\pm$ 0.0007 & 0.0255 $\pm$ 0.0342 & 64.30 $\pm$ 1.25  \\
& Fine-tuning   & 0.0157 $\pm$ 0.0104 & 0.9999 $\pm$ 0.0002 & 0.0007 $\pm$ 0.0007 & 0.0255 $\pm$ 0.0342 & 10.42 $\pm$ 0.75  \\
\cmidrule{2-7}
& Ours (CFlow)   & 0.0162 $\pm$ 0.0136 & 0.9999 $\pm$ 0.0002 & 0.0155 $\pm$ 0.0107 & 0.0385 $\pm$ 0.0375 & 50.07 $\pm$ 1.56 \\
\midrule

\multirow{3}{*}{MNIST} 
& Retrain (GT) & 0.0004 $\pm$ 0.0000 & 0.9861 $\pm$ 0.0098 & 0.0050 $\pm$ 0.0012 & 0.0108 $\pm$ 0.0009 & 300.00 $\pm$ 15.00 \\ 
 
& Fine-tuning   & 0.0039 $\pm$ 0.0004 & 0.9551 $\pm$ 0.0167 & 0.0143 $\pm$ 0.0032 & 0.0214 $\pm$ 0.0028 & 92.86 $\pm$ 6.12 \\
\cmidrule{2-7}
& Ours (CFlow) & 0.0020 $\pm$ 0.0003 & 0.9673 $\pm$ 0.0153 & 0.0005 $\pm$ 0.0005 & 0.0015 $\pm$ 0.0003 & 158.74 $\pm$ 11.34\\
\bottomrule

\multirow{3}{*}{CIFAR-10} 
& Retrain (GT) & 0.0056 $\pm$ 0.0004 & 0.8920 $\pm$ 0.0000 & 0.1127 $\pm$ 0.0078 & 0.1546 $\pm$ 0.0073 & \hspace{0.1cm} 802.37 $\pm$ 291.54 \\
& Fine-tuning   & 0.0077 $\pm$ 0.0016 & 0.9005 $\pm$ 0.0068 & 0.2157 $\pm$ 0.0095 & 0.2401 $\pm$ 0.0065 & 252.89 $\pm$ 18.72 \\

\cmidrule{2-7}
& Ours (CFlow)   & 0.0064 $\pm$ 0.0005 & 0.8847 $\pm$ 0.0077 & 0.1704 $\pm$ 0.0125 & 0.1748 $\pm$ 0.0109 & 427.15 $\pm$ 34.68 \\
\bottomrule
\end{tabular}
}
\end{table*}

\newpage
\section{Results}
We evaluate \textit{ContinualFlow} on structured 2D benchmarks and image domains to assess its ability to unlearn specific content while preserving generative performance. Starting from a model \( G_\theta \) trained on \( \mathcal{D}_{\text{full}} \), each task targets retention of \( \mathcal{D}_{\text{retain}} \) while suppressing \( \mathcal{D}_{\text{forget}} \), with performance evaluated against fine-tuning and retraining baselines. Our 2D settings include: \textbf{(1)} \textit{Circles}, keeping the inner ring; \textbf{(2)} \textit{Moons}, lower arc retained; \textbf{(3)} \textit{Gaussians}, odd-numbered clusters; and \textbf{(4)} \textit{Checkerboard}, top row and last column removed. Table~\ref{tab:results} summarizes the averaged results, with full metrics available in Appendix~\ref{apd:extensive_res}. We achieve performance comparable to retraining across diverse settings, including MNIST digit removal and CIFAR-10 class suppression.
\begin{figure}[H]
    \centering
    \includegraphics[width=0.47\textwidth]{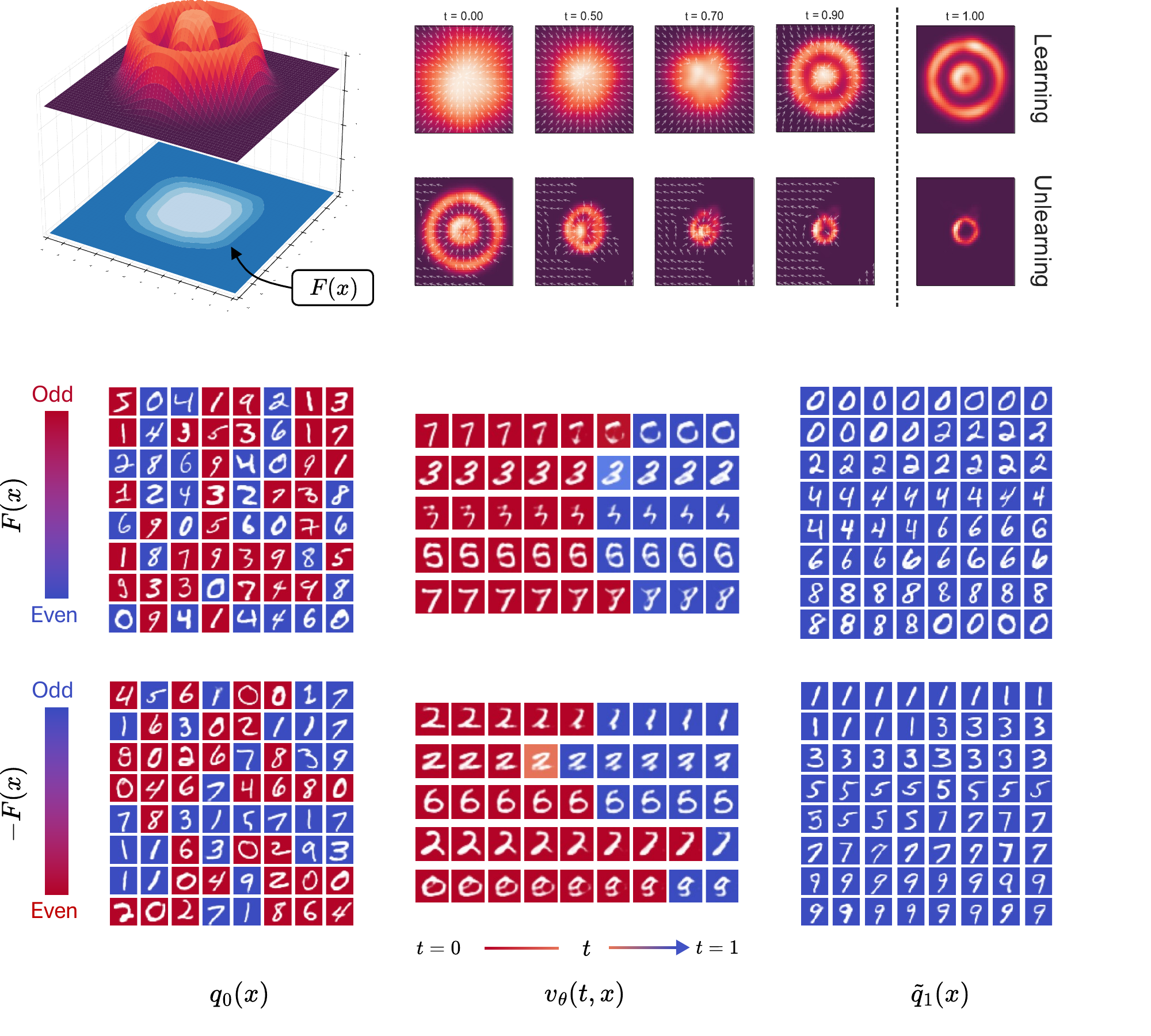}
    \caption{\textbf{Top:} 2D Circles. A model trained on both rings \(\mathcal{D}_{\text{full}}\) is guided by \(F(x)\) (blue surface) to suppress the outer ring \(\mathcal{D}_{\text{forget}}\) and generate inner-ring.  
\textbf{Bottom:} MNIST. Starting from a model trained on all digits. The top row uses \(F(x)\) to guide trajectories toward even digits (\(\mathcal{D}_{\text{retain}}\)), suppressing odds; the bottom row inverts the energy, redirecting flow to generate odd digits.} \label{fig:mnist-inversion} \end{figure} 

In Figure~\ref{fig:mnist-inversion}, we illustrate two settings. In the 2D Circles task (top), a model trained on both rings (\(\mathcal{D}_{\text{full}}\)) is guided by an energy function to suppress the outer ring (\(\mathcal{D}_{\text{forget}}\)), yielding a new model that generates only the inner ring (\(\mathcal{D}_{\text{retain}}\)). 

A distinctive property of our framework is the composability of energy functions, potentially allowing modular control over unlearning dynamics; in this work, we specifically investigate their invertibility, showing how reversing energy guidance recovers forgotten content without direct access. As an illustrative example on MNIST (Figure~\ref{fig:mnist-inversion}, bottom), we start with a model trained on all digits and apply \(F(x)\)  which penalizes odd digits, retaining evens. We then invert this energy and reapply our method, recovering a model that suppresses even digits and regenerates the odd class—despite no longer accessing its samples.  Reversing the energy reorients the learned flow toward \(\mathcal{D}_{\text{forget}}\), recovering this distribution without explicit sample access. This unlocks key capabilities for privacy-aware generative modeling, including augmenting sensitive classes for fairness evaluation, dynamically modulating access to protected regions, and enabling traceable, reversible unlearning for accountability in compliance-critical settings.

\vspace{-3mm}
\section{Conclusion}
We recast generative unlearning as a distributional transport problem and introduce \textit{ContinualFlow}, a framework that reshapes generative behavior by modulating flow trajectories with energy-based weights to suppress undesired regions. Through the integration of soft importance weights into training, it enables targeted unlearning without explicit forget sets or full retraining. This work frames unlearning not as an intervention on model weights but as a distributional shift through modulation of generative flows.

\noindent\textbf{Future Work.} A key challenge is learning energy functions that faithfully align with the forget distribution, as this alignment affects unlearning performance. Beyond class-level or binary proxies, extending to semantic, or geometric formulations is a promising direction. Moreover, the compositionality of energy functions can support modular objectives, particularly in continual and multi-stage unlearning.

\clearpage
\bibliography{main}
\bibliographystyle{icml2024}

\newpage
\appendix
\onecolumn

\begin{table*}[t]
\centering
\footnotesize
\setlength{\tabcolsep}{6pt}
\renewcommand{\arraystretch}{1.0}
\caption{Summary of symbols used throughout the paper.}
\label{tab:notation}
\begin{tabular}{ll|ll}
\toprule
\textbf{Symbol} & \textbf{Description} & \textbf{Symbol} & \textbf{Description} \\
\midrule
$q_0(x)$ & Source distribution (e.g., model output or prior) &
$q_f(x)$ & Forget distribution (to be unlearned, unknown) \\
$F(x)$ & Energy function, $F(x) \propto -\log q_f(x)$ &
$\lambda$ & Suppression sensitivity factor \\
$\tilde{R}(x)$ & Unnormalized reweighted target &
$\tilde{q}_1(x)$ & Soft mass-subtracted target \\
$v_\theta(t, x)$ & Learned velocity field &
$u_t(x \mid x_0, x_1)$ & Target velocity\\
$p_t(x \mid x_0, x_1)$ & Interpolated conditional &
$\mathcal{L}_{\text{CFM}}$ & Flow Matching loss \\
$Z$ & Normalization constant & $\mathcal{L}_{\text{ERFM}}$ & Energy-Reweighted FM loss \\
\bottomrule
\end{tabular}
\end{table*}

\section{Additional Related Work on Generative Machine Unlearning}
\label{apd:related_work}

To contextualize recent advances in generative model unlearning, we organize key works by their approach and focus area, emphasizing methods that directly tackle unlearning or targeted removal in generative settings.

\textbf{Output Suppression Methods.} These techniques act during or after generation to prevent disallowed content without modifying the model’s underlying parameters. A common example is post-hoc filtering, where language or image models employ classifiers or rule-based systems to detect and block sensitive or restricted outputs. For instance, \citet{gandikota2023erasing} propose Safe Latent Diffusion (SLD), which augments a diffusion model with nudity and violence detectors that steer the denoising trajectory away from generating Not Safe for Work (NSFW) or graphic content. While such inference-time guardrails significantly reduce the likelihood of unwanted outputs, they can often be dialed back or circumvented and may distort the output if applied too aggressively \citep{schramowski2023safe}. In LLMs, suppression is typically implemented through controlled decoding or fine-tuned refusal behaviors that steer models away from certain topics. Techniques such as prompt-based “safe completion” or RLHF operate as output-level controls. While effective at reducing harmful generations, they act more as alignment tools than true unlearning: the underlying knowledge remains encoded, and suppressed content may still be elicited through alternative prompts \citep{gandikota2023erasing}. 

These limitations motivate the complementary class of model-editing approaches discussed next.

\noindent \textbf{Model Editing and Target Removal Techniques.} Beyond suppression, another line of work seeks to edit model parameters directly to remove undesired concepts. In diffusion models, \citet{kumari2023ablating} propose a redirection approach, re-mapping target concepts to generic anchors to suppress stylistic outputs, for instance transforming “in the style of Artist X” into a neutral painting style. An alternative to suppression is to directly modify the model’s internal parameters to remove undesired concepts. However, model patching often risks collateral damage, degrading unrelated capabilities. To mitigate this, \citet{zhang2024defensive} and \citet{gandikota2024unified} introduce adversarial preservation losses that regularize neuron-level effects during unlearning. Fine-tuning on counterfactual text with KL regularization, or ablating neurons tied to memorized content, are common strategies for unlearning in LLMs. 

These challenges highlight a core limitation of current editing approaches: they rely on brittle interventions that often lack robustness, theoretical guarantees, and generalizability across tasks or modalities \citep{wu2023depn, xu2025relearn}.

\noindent \textbf{Continual and Online Unlearning.} Continual unlearning extends the classical unlearning setup to scenarios where deletions must be performed iteratively, without disrupting the model’s overall performance or utility. \citet{heng2023selective} propose Selective Amnesia, which leverages continual learning tools such as Elastic Weight Consolidation \citep{kirkpatrick2017overcoming} and retained data replay to remove target concepts while preserving overall model performance. However, these methods typically address fixed, one-shot unlearning tasks and rely on access to the data to be removed. \citet{thakral2025continual} extend this to the incremental setting, introducing a formal definition of continual generative unlearning. They show that repeated removals can cause \textit{generalization erosion}, where the model's output quality deteriorates even for unrelated prompts.

Continual unlearning remains an open challenge, raising questions about partial forgetting, compositionality, and robustness under evolving constraints. Our approach addresses this by enabling flow-based updates toward new targets while softly subtracting prior content, either using direct access to data distributions or via modular energy proxies. By design, this mechanism accommodates incremental and modular changes without compromising overall model coherence. This compositionality makes it especially well-suited for real-world deployments where unlearning needs to be flexible, efficient, and repeatable. A summary of the symbols used throughout this section is provided in Table~\ref{tab:notation}.

\newpage
\section{Flow Matching Toward Soft Mass-Subtracted Distributions.}
\label{apd:proof}

\subsection{Proof of Theorem ~\ref{thm:erfm_equiv}}

Let \( q_0(x) \) be a known base density, and let \( F(x) \propto -\log q_f(x) \) be an energy function for an unknown forget distribution \( q_f(x) \). We define the soft mass-subtracted target as
\[
\tilde{R}(x) \propto q_0(x) \cdot \sigma(-\lambda F(x)),
\]
and let \( \tilde{q}_1(x) = \tilde{R}(x)/Z \) denote its normalized version.

The objective is to show that the ERFM objective yields gradients equivalent, up to a positive scalar, to those of Conditional Flow Matching toward \( \tilde{q}_1 \).

Restating the ERFM loss:
\[
\mathcal{L}_{\mathrm{ERFM}}(\theta) =
\mathbb{E}_{x_0, x_1 \sim q_0;\, t \sim \mathcal{U}[0,1];\, x \sim p_t(x \mid x_0, x_1)}
\left[ \sigma(-\lambda F(x_1)) \cdot
\left\| v_\theta(t, x) - u_t(x \mid x_0, x_1) \right\|^2 \right].
\]

By definition of \( \tilde{q}_1 \), we can express the sampling as a reweighted importance sampling from \( q_0 \):
\[
\frac{\tilde{q}_1(x_1)}{q_0(x_1)} \propto \sigma(-\lambda F(x_1)).
\]

Hence, the CFM loss from \( q_0 \) to \( \tilde{q}_1 \) can be written as:
\[
\mathcal{L}_{\text{CFM}}^{q_0 \to \tilde{q}_1}(\theta) =
\mathbb{E}_{x_0, x_1 \sim q_0;\, t, x} 
\left[ \frac{\tilde{q}_1(x_1)}{q_0(x_1)} \cdot
\left\| v_\theta(t, x) - u_t(x \mid x_0, x_1) \right\|^2 \right],
\]
\[
= C^{-1} \cdot \mathcal{L}_{\mathrm{ERFM}}(\theta)
\quad \text{for some } C > 0,
\]
which implies:
\[
\nabla_\theta \mathcal{L}_{\mathrm{ERFM}}(\theta) = C \cdot \nabla_\theta \mathcal{L}_{\mathrm{CFM}}^{q_0 \to \tilde{q}_1}(\theta),
\]
as claimed in Theorem~\ref{thm:erfm_equiv}.
\qed

\subsection{Classifier-Based Energy Approximation}

Suppose we are given a probabilistic binary classifier $C(x) \in [0, 1]$ trained to distinguish samples from the base distribution $q_0(x)$ and an unlearned class $q_f(x)$. Under Bayes-optimality, the classifier estimates:
\[
C(x) = \mathbb{P}(x \in q_f \mid x) = \frac{q_f(x)}{q_0(x) + q_f(x)}
\]
We aim to show that such a classifier defines a valid energy function for our method.

\begin{proposition}
Let $C(x)$ be the Bayes-optimal probabilistic classifier distinguishing $q_f$ from $q_0$. Then the logit score
\[
F(x) := -\log\left( \frac{C(x)}{1 - C(x)} \right)
\]
defines an energy function such that the soft-mass subtracted target
\[
\tilde{R}(x) \propto q_0(x) \cdot \sigma(-\lambda F(x))
\]
is equivalent to
\[
\tilde{R}(x) \propto q_0(x) \cdot \frac{(1 - C(x))^\lambda}{(1 - C(x))^\lambda + C(x)^\lambda}
\]
\end{proposition}

\begin{proof}
By the classifier definition, we have:
\[
\frac{C(x)}{1 - C(x)} = \frac{q_f(x)}{q_0(x)} \quad \Rightarrow \quad F(x) = -\log\left( \frac{q_f(x)}{q_0(x)} \right)
\]
Hence:
\[
\sigma(-\lambda F(x)) = \frac{1}{1 + \left( \frac{C(x)}{1 - C(x)} \right)^\lambda}
= \frac{(1 - C(x))^\lambda}{(1 - C(x))^\lambda + C(x)^\lambda}
\]
This weight decreases as the classifier becomes more confident that $x$ belongs to $q_f$, suppressing contributions in the loss, and recovering the same behavior as energy-based reweighting.
\end{proof}

\section{Appendix: Background on Optimal Transport Flow Matching with Empirical Distributions}
\label{apd:otcfm}

\paragraph{Optimal Transport Foundations.}
Given two probability distributions \( q_0 \) and \( q_1 \) over \( \mathbb{R}^d \), the \textit{static} 2-Wasserstein distance between them is defined as:
\begin{equation}
    W_2^2(q_0, q_1) = \inf_{\pi \in \Pi(q_0, q_1)} \int_{\mathbb{R}^d \times \mathbb{R}^d} \|x_0 - x_1\|^2 \, d\pi(x_0, x_1),
\end{equation}
where \( \Pi(q_0, q_1) \) denotes the set of all couplings with marginals \( q_0 \) and \( q_1 \).

The \textit{dynamic} formulation of optimal transport (Benamou-Brenier \citep{Benamou2021OptimalTM}) seeks a time-varying density \( p_t \) and vector field \( u_t \) that minimize:
\begin{equation}
    \inf_{p_t, u_t} \int_0^1 \int_{\mathbb{R}^d} p_t(x) \|u_t(x)\|^2 \, dx\, dt,
\end{equation}
subject to the continuity equation
\[
\frac{\partial p_t}{\partial t} + \nabla \cdot (p_t u_t) = 0,
\]
with boundary conditions \( p_0 = q_0 \), \( p_1 = q_1 \).

\paragraph{Flow Matching with Empirical Marginals.}
In the CFM framework~\citep{tong2023improving}, the velocity field \( v_\theta(t,x) \) is trained to match an analytically defined vector field \( u_t(x \mid z) \), using samples from a conditional distribution \( p_t(x \mid z) \). This is generalized further in the OT-CFM formulation, where the coupling \( z = (x_0, x_1) \) is sampled according to an optimal transport plan \( \pi(x_0, x_1) \), instead of an independent product \( q_0(x_0) q_1(x_1) \).

Given such a coupling, the interpolation and associated velocity fields are:
\begin{align}
    \psi_t(x_0, x_1) &= (1 - t)x_0 + t x_1, \\
    p_t(x \mid x_0, x_1) &= \mathcal{N}(x \mid \psi_t(x_0, x_1), \sigma^2 \mathbb{I}), \\
    u_t(x \mid x_0, x_1) &= x_1 - x_0.
\end{align}

\paragraph{OT-CFM Justification.}
Proposition 3.4 of~\citet{tong2023improving} proves that when the coupling \( \pi(x_0, x_1) \) is the OT plan and \( \sigma^2 \rightarrow 0 \), the marginal flow field \( u_t(x) \) minimizes the dynamic OT objective between \( q_0 \) and \( q_1 \). Notably, this formulation imposes no assumption that \( q_0 \) must be a Gaussian or that its density is known.

This formulation enables the construction of a transport path between the empirical distribution of samples generated by a pretrained model \( G_\theta \), denoted \( q_0 \), and the empirical distribution over the retained dataset \( \mathcal{D}_{\text{retain}} \), denoted \( q_1 \), without requiring access to their explicit densities. Leveraging this coupling, we can train a new velocity field \( v_{\tilde{\theta}} \) to model the flow between \( q_0 \) and \( q_1 \), thereby avoiding the need to retrain from a fixed prior such as a standard Gaussian. This yields a modular approach for implementing targeted unlearning and supporting continual distributional updates.

\paragraph{Minibatch Approximation.}
In practical settings where the exact OT plan is computationally expensive, minibatch OT~\citep{Nguyen2021ImprovingMO} can be used to approximate \( \pi(x_0, x_1) \). Empirically, this yields competitive flows and maintains convergence benefits without incurring the overhead of solving full OT across the dataset.


\section{Extended Results and Implementation Details}
\label{apd:extensive_res}

To rigorously evaluate the performance of targeted unlearning, we report results across three complementary axes: \textbf{retention}, \textbf{forgetting}, and \textbf{efficiency}. This section formally defines each metric, motivates its inclusion, and provides implementation details for reproducibility.

\paragraph{Maximum Mean Discrepancy (MMD$_\text{retain}$).}  
We compute MMD$_\text{retain}$ between generated samples and a held-out subset of retained data to assess distributional alignment. Formally, for two sets $X = \{x_i\}$ and $Y = \{y_j\}$ of generated and retained samples, we estimate:
\[
\text{MMD}^2(X, Y) = \frac{1}{n^2} \sum_{i,i'} k(x_i, x_{i'}) + \frac{1}{m^2} \sum_{j,j'} k(y_j, y_{j'}) - \frac{2}{nm} \sum_{i,j} k(x_i, y_j),
\]
using an RBF kernel with fixed bandwidth $\sigma = 1.0$. We report the mean and standard deviation of this value across multiple randomized evaluations with $n = m = 1000$.

\paragraph{Retention Accuracy.}  
We evaluate a pretrained binary classifier on real retained samples and report the percentage of correctly predicted labels. This measures the extent to which information about the preserved content is retained by the model, and serves as a proxy for functionality preservation.

\paragraph{Forget Rate.}  
To assess how effectively the model suppresses the forgotten content, we measure the proportion of generated samples that the classifier assigns to the forget class. A low forget rate indicates fewer instances of undesired generation and stronger unlearning.

\paragraph{Leakage Score.}  
While the forget rate captures coarse-level suppression, the leakage score quantifies fine-grained semantic resemblance to forgotten content. Specifically, we compute the classifier’s average confidence on the forget class over generated samples. 

\paragraph{Training Time.}  
Training time is recorded as wall-clock duration required to complete a fixed number of optimization steps. Inference time is measured as the average per-sample generation time (in milliseconds) across 5000 samples using $n_\text{steps} = 10$ flow integration steps.

Table~\ref{tab:extended2d} reports a comprehensive comparison of unlearning performance across four 2D datasets: \textit{Circles}, \textit{Checkerboard}, \textit{Moons}, and \textit{6 Gaussians}. Each experiment is repeated with three independent classifier runs to account for variability, and all reported values represent the mean and standard deviation across these trials.

\begin{table*}[h]
\centering
\small
\caption{Comparison of unlearning performance across methods and 2D datasets.}
\label{tab:extended2d}
\resizebox{\textwidth}{!}{
\begin{tabular}{llcccccc}
\toprule
\textbf{Dataset} & \textbf{Method} 
& \multicolumn{2}{c}{\textbf{Retention}} 
& \multicolumn{2}{c}{\textbf{Forgetting}} 
& \multicolumn{2}{c}{\textbf{Efficiency}} \\
& & MMD ($\downarrow$) & Accuracy ($\uparrow$) 
  & Forget Rate ($\downarrow$) & Leakage ($\downarrow$) 
  & Train Time (s) & Inference (ms) \\
\midrule

\multirow{3}{*}{Circles} 
& Retrain (GT)    & 0.0011 $\pm$ 0.0002 & 1.0000 $\pm$ 0.0000 & 0.0000 $\pm$ 0.0000 & 0.0840 $\pm$ 0.0593 & 10.50 $\pm$ 0.50 & 0.004 ms \\
& Fine-tuning     & 0.0011 $\pm$ 0.0002 & 1.0000 $\pm$ 0.0000 & 0.0000 $\pm$ 0.0000 & 0.0840 $\pm$ 0.0593 & 1.60 $\pm$ 0.20  & 0.005 ms \\
& Ours (CFlow)    & 0.0021 $\pm$ 0.0002 & 1.0000 $\pm$ 0.0000 & 0.0173 $\pm$ 0.0106 & 0.1000 $\pm$ 0.0675 & 8.20 $\pm$ 0.30  & 0.004 ms \\
\midrule

\multirow{3}{*}{Checkerboard} 
& Retrain (GT)    & 0.0136 $\pm$ 0.0002 & 0.9996 $\pm$ 0.0001 & 0.0004 $\pm$ 0.0002 & 0.0006 $\pm$ 0.0001 & 13.50 $\pm$ 0.40 & 0.004 ms \\
& Fine-tuning     & 0.0136 $\pm$ 0.0002 & 0.9996 $\pm$ 0.0001 & 0.0004 $\pm$ 0.0002 & 0.0006 $\pm$ 0.0001 & 2.20 $\pm$ 0.30  & 0.005 ms \\
& Ours (CFlow)    & 0.0063 $\pm$ 0.0004 & 0.9996 $\pm$ 0.0002 & 0.0002 $\pm$ 0.0000 & 0.0002 $\pm$ 0.0001 & 9.30 $\pm$ 0.50  & 0.004 ms \\
\midrule

\multirow{3}{*}{Moons} 
& Retrain (GT)    & 0.0179 $\pm$ 0.0002 & 1.0000 $\pm$ 0.0000 & 0.0005 $\pm$ 0.0001 & 0.0142 $\pm$ 0.0039 & 17.20 $\pm$ 0.60 & 0.005 ms \\
& Fine-tuning     & 0.0179 $\pm$ 0.0002 & 1.0000 $\pm$ 0.0000 & 0.0005 $\pm$ 0.0001 & 0.0142 $\pm$ 0.0039 & 2.90 $\pm$ 0.30  & 0.006 ms \\
& Ours (CFlow)    & 0.0194 $\pm$ 0.0009 & 1.0000 $\pm$ 0.0000 & 0.0142 $\pm$ 0.0020 & 0.0199 $\pm$ 0.0007 & 13.60 $\pm$ 0.50 & 0.004 ms \\
\midrule

\multirow{3}{*}{6 Gaussians} 
& Retrain (GT)    & 0.0303 $\pm$ 0.0001 & 1.0000 $\pm$ 0.0000 & 0.0018 $\pm$ 0.0003 & 0.0031 $\pm$ 0.0015 & 23.10 $\pm$ 0.80 & 0.005 ms \\
& Fine-tuning     & 0.0303 $\pm$ 0.0001 & 1.0000 $\pm$ 0.0000 & 0.0018 $\pm$ 0.0003 & 0.0031 $\pm$ 0.0015 & 3.70 $\pm$ 0.40  & 0.005 ms \\
& Ours (CFlow)    & 0.0370 $\pm$ 0.0010 & 1.0000 $\pm$ 0.0000 & 0.0302 $\pm$ 0.0083 & 0.0338 $\pm$ 0.0083 & 19.00 $\pm$ 0.60 & 0.004 ms \\
\bottomrule
\end{tabular}
}
\end{table*}

\subsection{Experimental Settings}
\textbf{MNIST:} The quantitative results reported in Table~\ref{tab:results} are based on a binary MNIST task where the retain set \(D_{\text{retain}}\) includes even digits (0, 2, 4, 6, 8) and the forget set \(D_{\text{forget}}\) includes odd digits (1, 3, 5, 7, 9).

In Figure~\ref{fig:mnist_full}, we visualize generation behavior in a different, more constrained setting. Here, \(D_{\text{retain}}\) consists solely of digit ``0'', and \(D_{\text{forget}}\) includes digits ``1--9''. As shown in the top row, increasing the energy scaling parameter \(\lambda\) progressively suppresses the forget set, leading the model to generate only digit ``0''. In the bottom row, the energy function is reversed to penalize digit ``0'' while retaining digits ``1--9''. As \(\lambda\) increases, generation shifts away from digit ``0'', demonstrating the model’s ability after training to exclude specific classes while preserving others.

\begin{figure}[hbtp]
    \centering
    \includegraphics[width=0.7\linewidth]{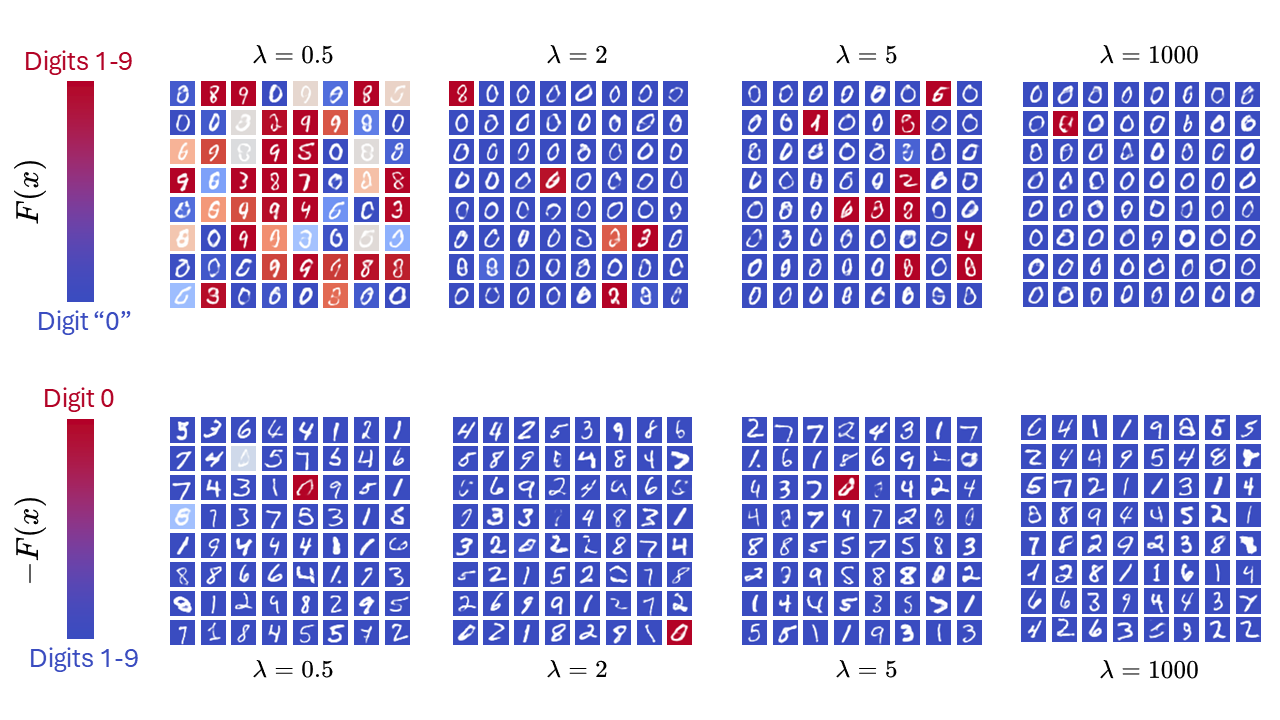}
        \caption{\textbf{Effect of suppression factor \(\lambda\).} \textbf{Top row:} The energy function \(F(z)\) is low on \(D_{\text{retain}}\) (digit ``0'') and high on \(D_{\text{forget}}\) (digits ``1--9''). As \(\lambda\) increases (\(\lambda \in \{0.5, 2, 5, 1000\}\)), the model progressively suppresses \(D_{\text{forget}}\), converging toward samples from \(D_{\text{retain}}\). \textbf{Bottom row:} Reversing \(F(z)\) to penalize digit ``0'' instead, the model excludes \(D_{\text{retain}}\) as \(\lambda\) grows. This demonstrates energy-guided modulation of generation in ContinualFlow.}
    \label{fig:mnist_full}
\end{figure}

\textbf{2D Distributions:} We begin by evaluating \textit{ContinualFlow} on a set of 2D synthetic distributions commonly used in generative modeling: Circles, Moons, Gaussians, and Checkerboard. Each task involves learning a mapping from a Gaussian base distribution to the target, followed by an unlearning phase where a specific region is suppressed via an energy function \(F(x)\). These experiments provide an interpretable setting to assess the trajectory-level behavior of the model during learning and unlearning. The corresponding averaged quantitative results for these tasks are reported in Table~\ref{tab:results}. Figure~\ref{fig:multi-row-energy-flow} visualizes the evolution of samples under both phases. For each dataset, the top row shows standard flow-based generation toward the full target distribution, while the bottom row shows the result of applying energy-reweighted unlearning to suppress designated regions. 

\begin{wrapfigure}{r}{0.5\linewidth}
    \centering
    \vspace{-1em}
    \includegraphics[width=\linewidth]{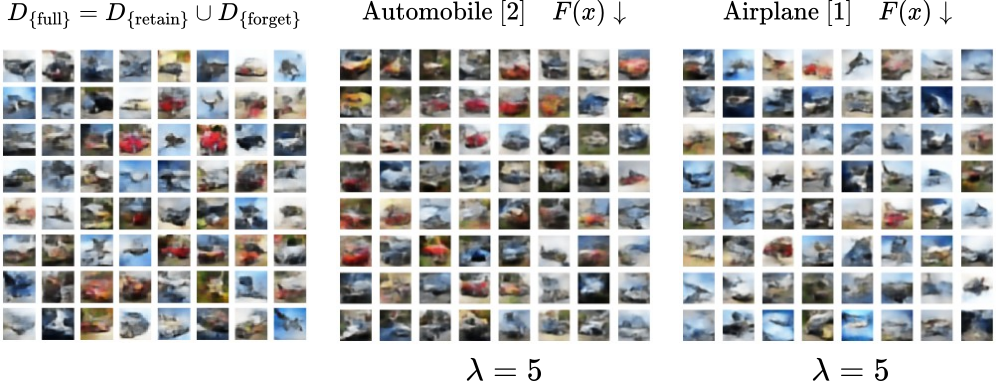}
    \caption{\textbf{CIFAR-10 latent unlearning.} Generation from the full distribution (left), and after suppressing all classes except \textit{automobile} (middle) and \textit{airplane} (right), which are assigned low energy.}
    \label{fig:cifar_flow}
\end{wrapfigure}

\textbf{CIFAR-10:} Finally, we evaluate \textit{ContinualFlow} on CIFAR-10 by operating in a 64-dimensional latent space derived from a pretrained autoencoder. This experiment tests the model’s ability to suppress specific semantic classes while retaining others, under limited latent capacity. The full distribution \(D_{\text{full}} = D_{\text{retain}} \cup D_{\text{forget}}\) consists of a subset of two CIFAR-10 classes: \textit{automobile} (1) and \textit{airplane} (2). The left column shows samples generated by a generative model trained on this full distribution. In the center column, the model is guided to retain only the \textit{automobile} class by assigning it low energy and suppressing all other modes. The right column mirrors this by retaining only the \textit{airplane} class.

\begin{figure}[h]
  \centering

  \begin{minipage}[c]{0.23\textwidth}
    \includegraphics[width=\textwidth]{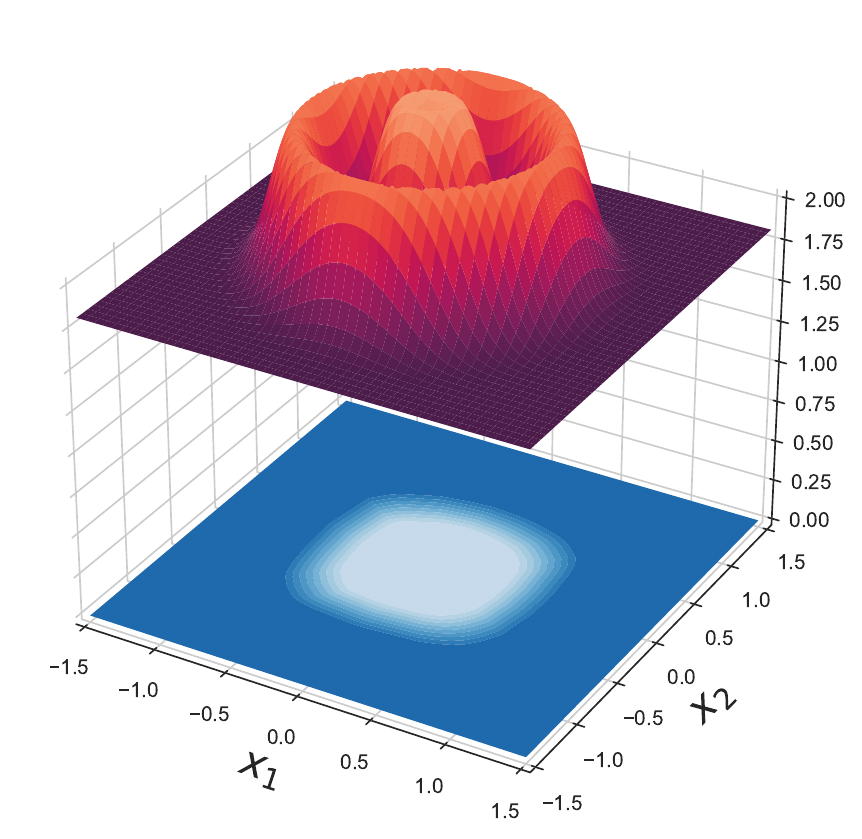}
  \end{minipage}%
  \hfill
  \begin{minipage}[c]{0.7\textwidth}
    \includegraphics[width=\textwidth]{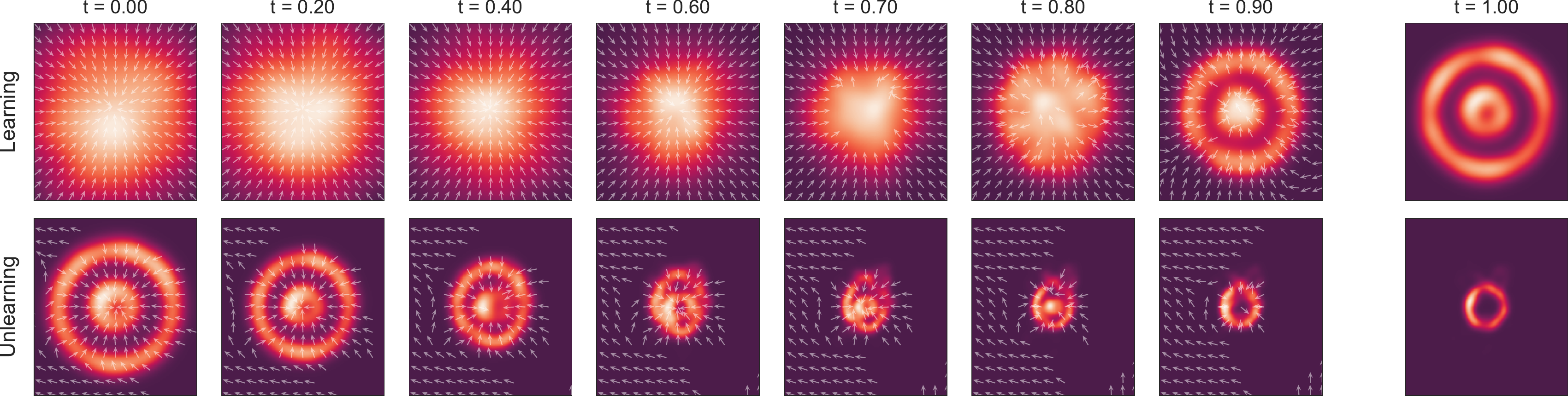}
  \end{minipage}

  \vspace{0.6em}  

  \begin{minipage}[c]{0.23\textwidth}
    \includegraphics[width=\textwidth]{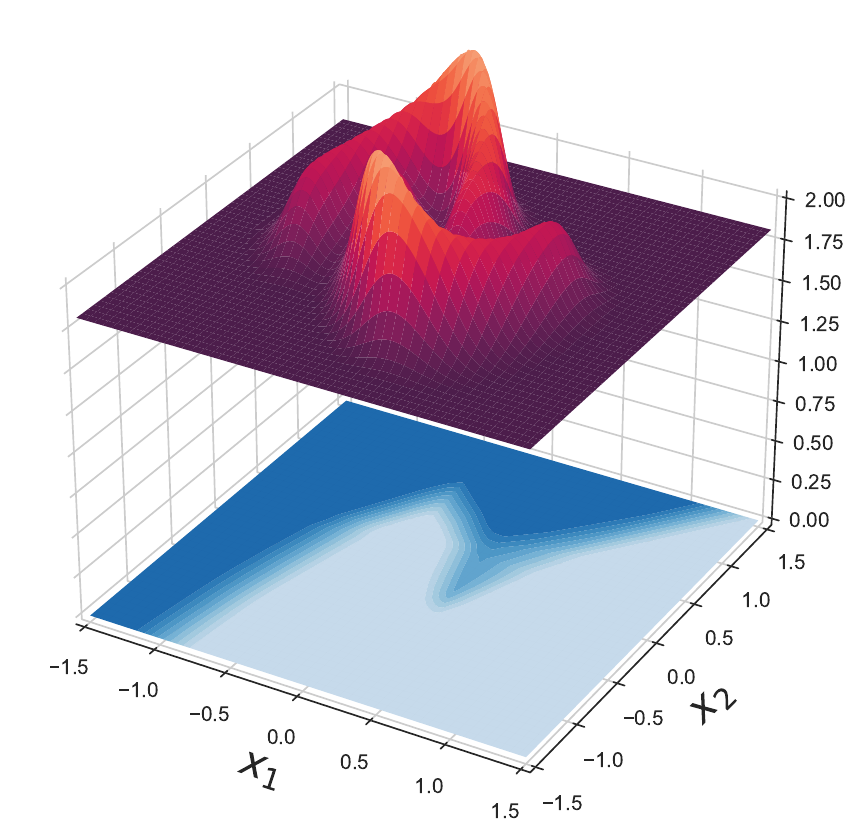}
  \end{minipage}%
  \hfill
  \begin{minipage}[c]{0.7\textwidth}
    \includegraphics[width=\textwidth]{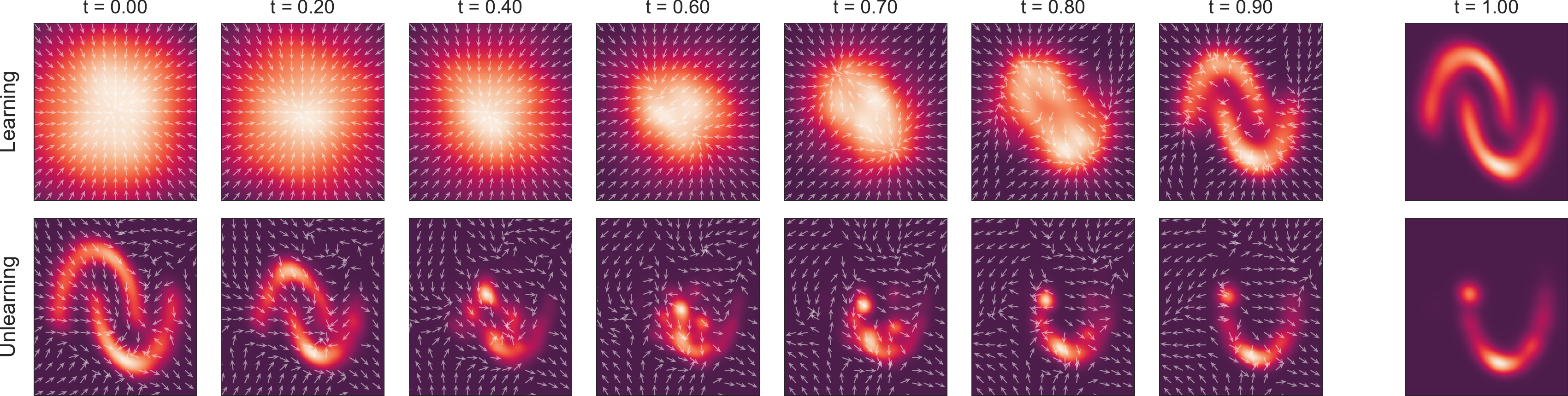}
  \end{minipage}

  \vspace{0.6em}

  \begin{minipage}[c]{0.23\textwidth}
    \includegraphics[width=\textwidth]{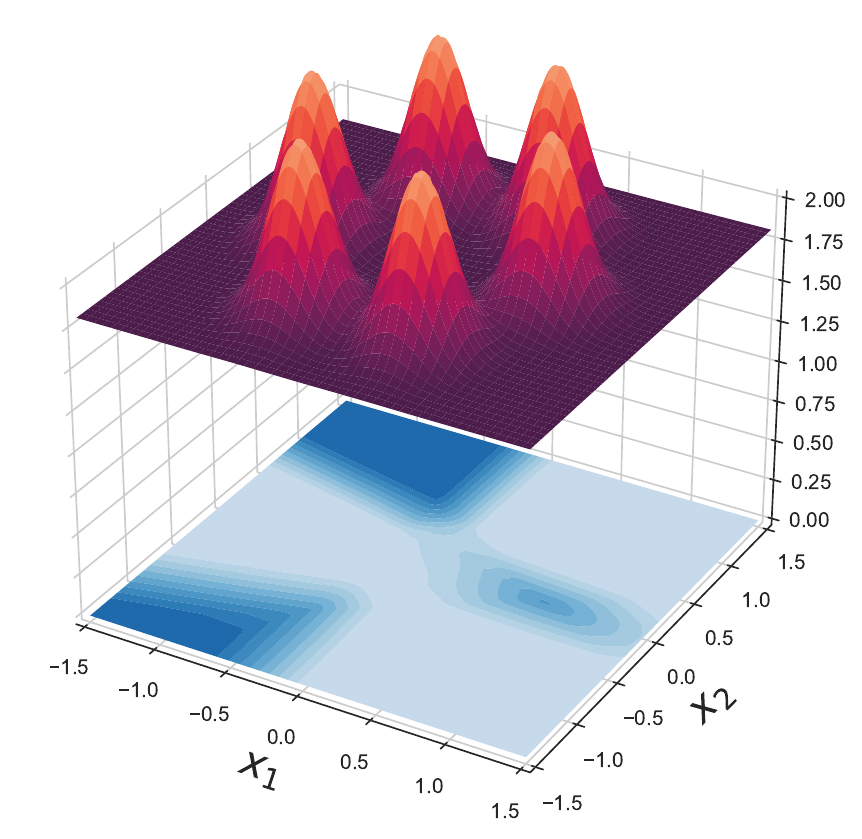}
  \end{minipage}%
  \hfill
  \begin{minipage}[c]{0.7\textwidth}
    \includegraphics[width=\textwidth]{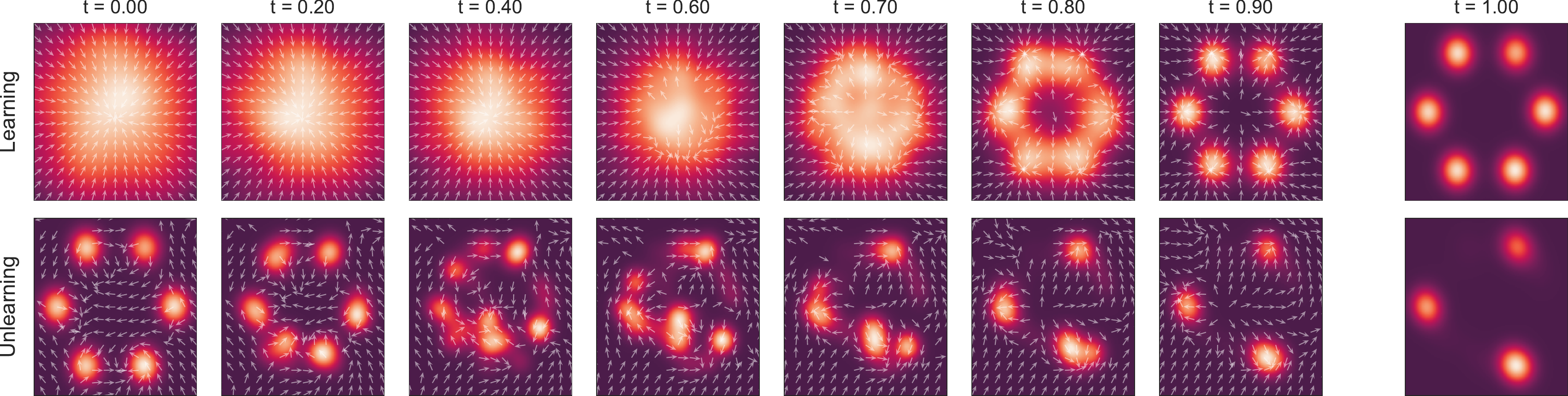}
  \end{minipage}

  \vspace{0.6em}

  \begin{minipage}[c]{0.23\textwidth}
    \includegraphics[width=\textwidth]{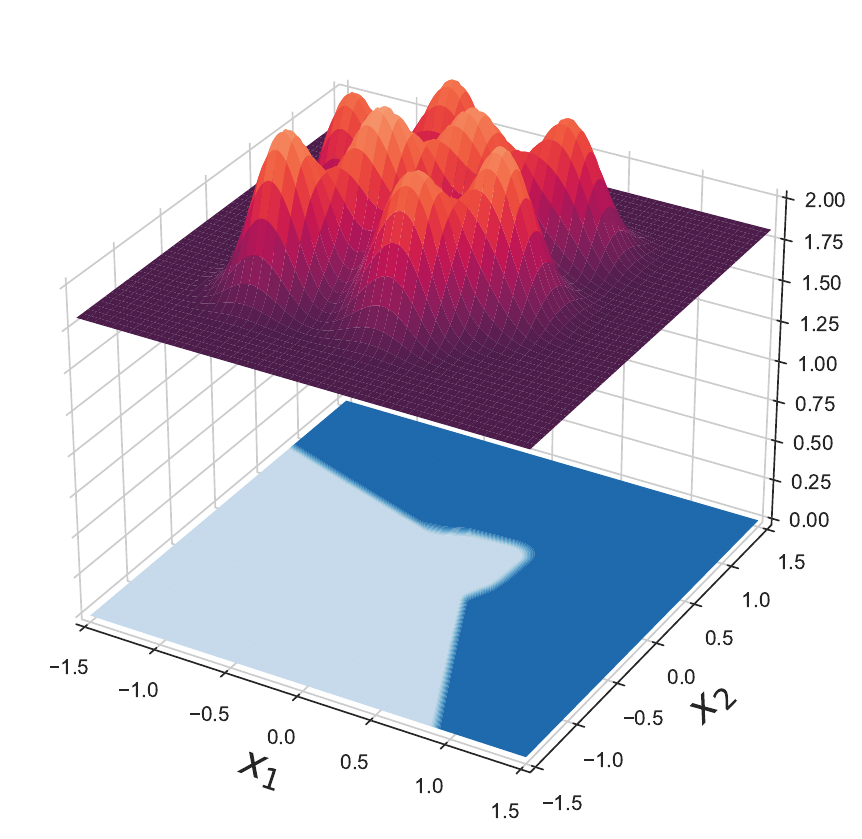}
  \end{minipage}%
  \hfill
  \begin{minipage}[c]{0.7\textwidth}
    \includegraphics[width=\textwidth]{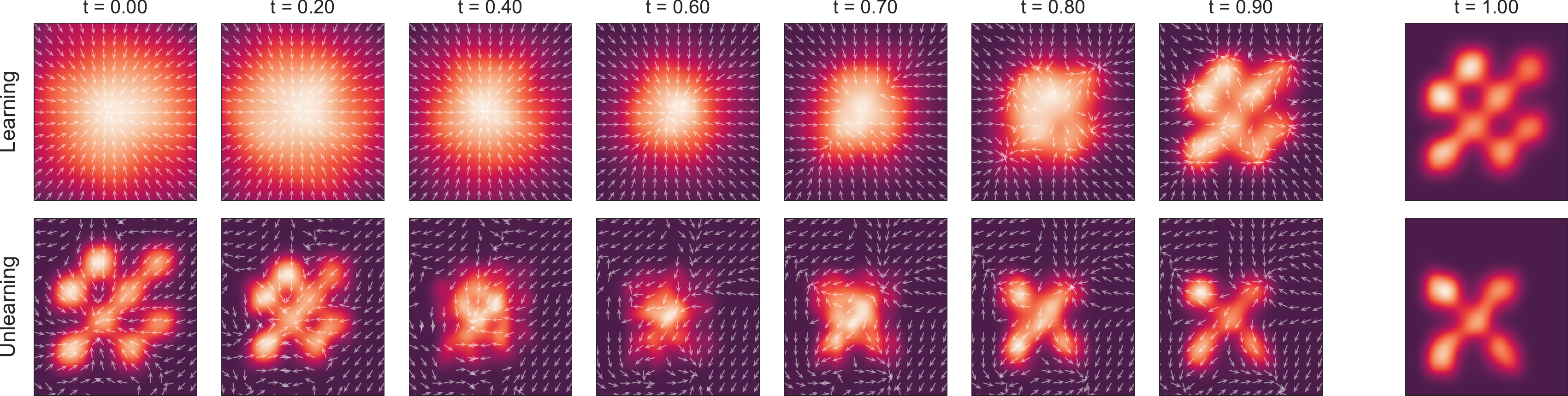}
  \end{minipage}

  \caption{The left panel shows the target density as a surface plot, with the corresponding energy function $F(x)$ rendered on the floor plane. The right panel illustrates the learned and unlearned trajectories for four 2D benchmarks: Circles, Moons, 6 Gaussians, and Checkerboard. The parameter $\lambda = 5$ is used for each experiment.}
  \label{fig:multi-row-energy-flow}
\end{figure}

\clearpage
\subsection{Hardware Specifications}
\label{apd:hardware}

All experiments were conducted using NVIDIA GPUs with CUDA-enabled PyTorch acceleration. For low-dimensional 2D synthetic benchmarks we employed a local workstation equipped with an NVIDIA GeForce RTX 3070 GPU and 8GB of VRAM. For experiments in the image domain, we utilized an NVIDIA L4 GPU with 24GB of VRAM, hosted in a cloud compute environment. This enabled the training of convolutional encoder–decoder pairs and larger latent flows required for modeling high-dimensional data distributions.

\begin{figure}[h]
\centering
\begin{tikzpicture}[
    font=\footnotesize, 
    col1/.style={text width=3.2cm, align=left, font=\bfseries, draw, minimum height=0.5cm, fill=gray!5, inner sep=2pt},
    col2/.style={text width=3.5cm, align=center, draw, minimum height=0.5cm, fill=blue!2, inner sep=2pt},
    col3/.style={text width=3.5cm, align=center, draw, minimum height=0.5cm, fill=blue!2, inner sep=2pt},
    header/.style={fill=gray!20, font=\bfseries\footnotesize, minimum height=0.5cm}
  ]

\matrix (table) [matrix of nodes, row sep=0.02cm, column sep=0pt,
    column 1/.style=col1,
    column 2/.style=col2,
    column 3/.style=col3,
    nodes in empty cells] {
  \node[header] {Property};       & \node[header] {2D Experiments}; & \node[header] {Image Experiments}; \\
  GPU            & NVIDIA RTX 3070 & NVIDIA L4 \\
  CUDA Version   & 12.4                    & 12.4 \\
  Driver Version & 552.22                  & 550.54.15 \\
  VRAM           & 8GB                     & 24GB \\
};

\end{tikzpicture}
\caption{Hardware specifications used for 2D and image-domain experiments.}
\end{figure}

\subsection{Architectural Details for Latent-Space Image Generation}
\label{apd:arch}

For image-based experiments on MNIST and CIFAR-10, we employ a modular autoencoder–flow pipeline composed of a convolutional encoder, a transposed-convolutional decoder, and a conditional latent flow model. The encoder compresses the input image into a compact latent vector $z \in \mathbb{R}^d$, which serves as a transport space for learning and unlearning tasks. Sampling proceeds by first generating latent codes from a base Gaussian via conditional flow, followed by energy-reweighted transformation, and decoding to the image domain.

\begin{figure}[h]
\centering
\scalebox{0.82}{
\begin{tikzpicture}[
    layer/.style={
        draw,
        minimum width=4cm,
        text width=4cm,
        minimum height=0.65cm,
        align=center,
        fill=blue!5,
        font=\footnotesize
    },
    header/.style={font=\bfseries\small},
    arrow/.style={->, thick},
    dashedarrow/.style={->, thick, dashed},
    node distance=0.4cm
  ]

\node[header] (m0) at (0, 0) {MNIST (Grayscale, 28$\times$28)};

\node[layer, below=0.4cm of m0] (m1) {Conv2D 1$\rightarrow$32\\4$\times$4, stride=2 + ReLU};
\node[layer, below=of m1] (m2) {Conv2D 32$\rightarrow$64\\4$\times$4, stride=2 + ReLU};
\node[layer, below=of m2] (m3) {Flatten + Linear\\3136$\rightarrow d$};
\node[layer, below=of m3, fill=gray!10] (mz) {Latent $z$};
\node[layer, below=of mz] (m4) {Linear $d\rightarrow$3136 + ReLU};
\node[layer, below=of m4] (m5) {Unflatten + ConvT\\64$\rightarrow$32 + ReLU};
\node[layer, below=of m5] (m6) {ConvT 32$\rightarrow$1 + Tanh};

\foreach \a/\b in {m1/m2, m2/m3, m3/mz, mz/m4, m4/m5, m5/m6}
    \draw[arrow] (\a) -- (\b);

\node[header, right=6.5cm of m0] (c_header) {CIFAR-10 (RGB, 32$\times$32)};
\node[layer, below=0.4cm of c_header] (c1) {2× Conv2D\\3$\rightarrow$64 + ReLU + MaxPool};
\node[layer, below=of c1] (c2) {Conv2D 64$\rightarrow$128 + ReLU + MaxPool};
\node[layer, below=of c2] (c3) {Conv2D 128$\rightarrow$256 + ReLU + MaxPool};
\node[layer, below=of c3] (c4) {Conv2D 256$\rightarrow$512 + ReLU};
\node[layer, below=of c4] (c5) {Flatten + Linear\\8192$\rightarrow d$};
\node[layer, below=of c5, fill=gray!10] (cz) {Latent $z$};
\node[layer, below=of cz] (c6) {Linear $d\rightarrow$8192 + Unflatten\\(512,4,4)};
\node[layer, below=of c6] (c7) {ConvT 512$\rightarrow$256 → Conv 256};
\node[layer, below=of c7] (c8) {ConvT 256$\rightarrow$128 → Conv 128};
\node[layer, below=of c8] (c9) {ConvT 128$\rightarrow$64 → Conv 64};
\node[layer, below=of c9] (c10) {Conv 64$\rightarrow$3 + Tanh};

\foreach \a/\b in {c1/c2, c2/c3, c3/c4, c4/c5, c5/cz, cz/c6, c6/c7, c7/c8, c8/c9, c9/c10}
    \draw[arrow] (\a) -- (\b);

\draw[dashedarrow] (mz.east) -- ++(1.5,0) node[midway, above] {\small Flow};
\draw[dashedarrow] (cz.east) -- ++(1.5,0) node[midway, above] {\small Flow};

\end{tikzpicture}
}
\caption{Comparative architecture diagram of the encoder–decoder models designed for MNIST and CIFAR-10 experiments.}
\end{figure}
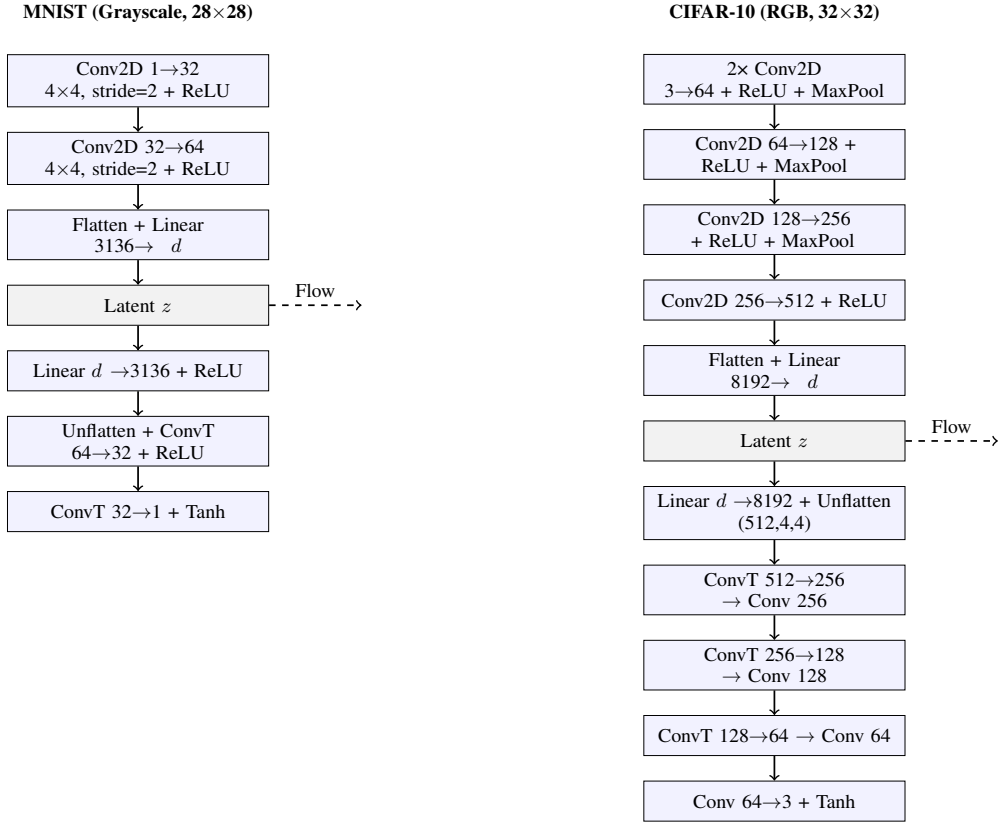


\clearpage

\end{document}